\documentclass[twoside,11pt]{article}

\usepackage{jmlr2e}
\usepackage{float}

\usepackage{amsfonts}
\usepackage{amsmath}
\usepackage{subcaption} 

\usepackage[ruled, linesnumbered]{algorithm2e}
\SetKw{Continue}{continue}
\SetKw{And}{and}
\makeatletter
\newcommand\fs@spaceruled{\def\@fs@cfont{\bfseries}\let\@fs@capt\floatc@ruled
  \def\@fs@pre{\vspace{5\baselineskip}\hrule height.8pt depth0pt \kern2pt}%
  \def\@fs@post{\kern2pt\hrule\relax}%
  \def\@fs@mid{\kern2pt\hrule\kern2pt}%
  \let\@fs@iftopcapt\iftrue}
\makeatother

\newcommand{\smin}{s_{\text{min}}}

\ShortHeadings{Labelling as an unsupervised learning problem}{Lyons and Perez Arribas}
\firstpageno{1}

\begin{document}

\title{Labelling as an unsupervised learning problem}

\author{\name Terry Lyons$^\ast\phantom{}^\dag$ \email terry.lyons@maths.ox.ac.uk \\
       \name Imanol Perez Arribas$^\ast\phantom{}^\dag$ \email imanol.perez@maths.ox.ac.uk \\
       \addr Mathematical Institute, University of Oxford (UK)$^\ast$\\		       
       \addr Alan Turing Institute, London (UK)$\phantom{}^\dag$
       }

\editor{}

\maketitle

\begin{abstract}

Unravelling hidden patterns in datasets is a classical problem with many potential applications. In this paper, we present a challenge whose objective is to discover nonlinear relationships in noisy cloud of points. If a set of point satisfies a nonlinear relationship that is unlikely to be due to randomness, we will \textit{label} the set with this relationship. Since points can satisfy one, many or no such nonlinear relationships, cloud of points will typically have one, multiple or no labels at all. This introduces the \textit{labelling problem} that will be studied in this paper.


The objective of this paper is to develop a framework for the labelling problem. We introduce a precise notion of a label, and we propose an algorithm to discover such labels in a given dataset, which is then tested in synthetic datasets. We also analyse, using tools from random matrix theory, the problem of discovering false labels in the dataset.

\end{abstract}



\section{Introduction}

Identifying hidden patterns in a given dataset is a problem of great interest \citep{CLUSTERINGAPPLICATIONS, CLUSTERINGAPPLICATIONS2, CLUSTERINGAPPLICATIONS3}, and as a consequence it has been extensively studied \citep[Chapter 14]{UNSUPERVISED}. This is precisely the objective of unsupervised learning \citep{UNSUPERVISEDOBJECTIVE} -- one has access to a dataset, and one wishes to find underlying structure hidden in the data.

The objective of this paper is to introduce the \textit{labelling problem}, an unsupervised learning problem where one intends to assign to points in a datasets \textit{labels}, an object that will be defined in Section \ref{sec:labelling}.

\subsection{Labelling}

A cloud of points will share a label if unreasonably many of them share a common relationship. This relationship will not be determined by the individual points, but by the cloud of points as a whole.

For example, if sufficient points in a Euclidean space lie on a hyperplane, the hyperplane captures a relationship between the points that is unlikely to be due to randomness. Therefore, one could \textit{label} the cloud of points with that hyperplane. Consequently, points could have multiple labels since they do not have to lie on a single hyperplane.

This linear setting is a particular case of the noisy nonlinear framework we propose in Section \ref{sec:labelling} -- the \textit{labelling problem}. In this case, points will be labelled together if they satisfy a nonlinear relationship that is unlikely to be due to randomness -- see Definition \ref{def:label}.

This is a natural problem to consider, since many real-life problems consist of finding relationships that are not mutually exclusive. For instance, a group of people may share the property that they all like a particular movie -- the movie would be a label assigned to this group of people -- but each of them may also like other movies, which would relate them to other groups of people. That is, any given individual will have multiple labels: the movies he or she likes. Then, one could try to find users that share the same labels -- that is, users that like the same movies.

Consider Figures \ref{fig:two circles} and \ref{fig:arbitrary conics first}. It is visually apparent that there are at least two labels in Figure \ref{fig:two circles},  and three labels in Figure \ref{fig:arbitrary conics first}, and that some points have more than one label.

\begin{figure}[h]
\centering
\begin{minipage}[t]{.45\textwidth}
  \centering
   \includegraphics[width=\textwidth]{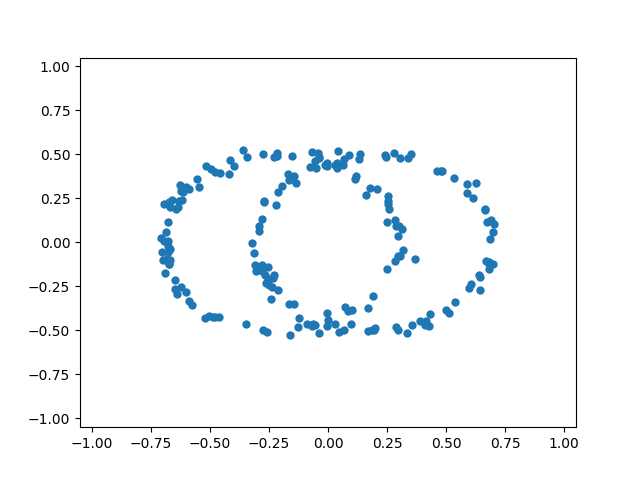}
    \caption{Two overlapping circles.}
	\label{fig:two circles}

\end{minipage}\hfill
\begin{minipage}[t]{.45\textwidth}
	\centering
   \includegraphics[width=\textwidth]{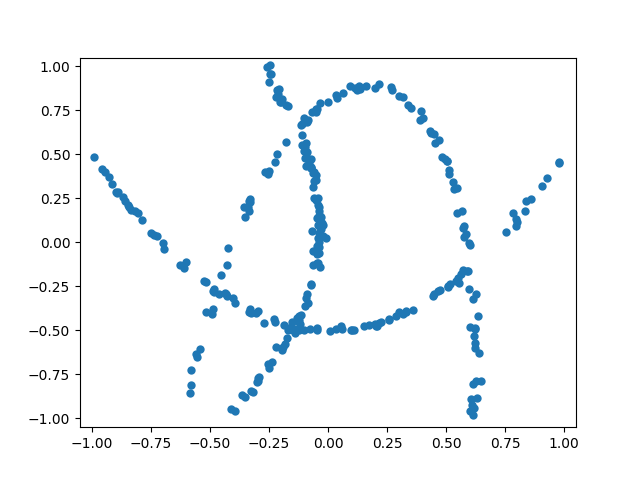}
    \caption{Three overlapping conics on the plane.}
	\label{fig:arbitrary conics first}
\end{minipage}
\end{figure}

\begin{figure}[h]
\begin{minipage}[t]{.45\textwidth}
  \centering
   \includegraphics[width=\textwidth]{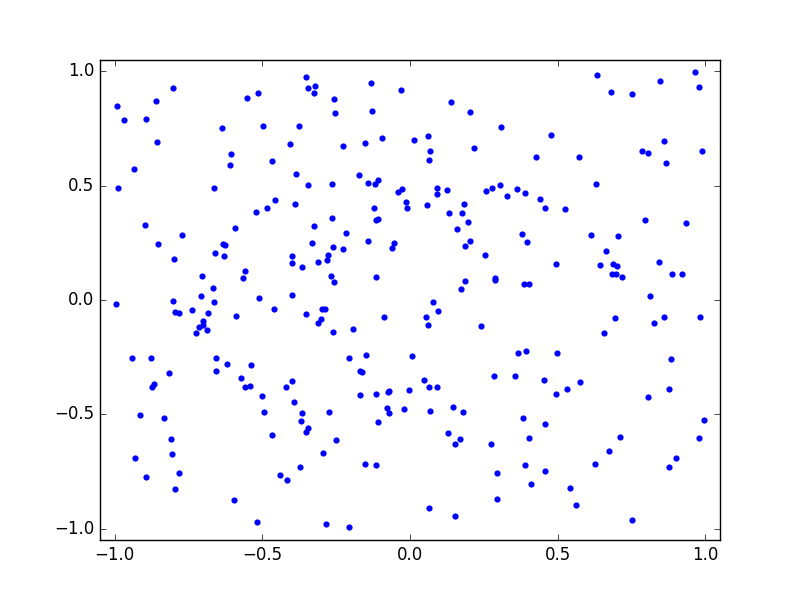}
    \caption{Figure \ref{fig:two circles} with low signal-to-noise ratio.}
	\label{fig:two circles noisy}

\end{minipage}\hfill
\begin{minipage}[t]{.45\textwidth}
	\centering
   \includegraphics[width=\textwidth]{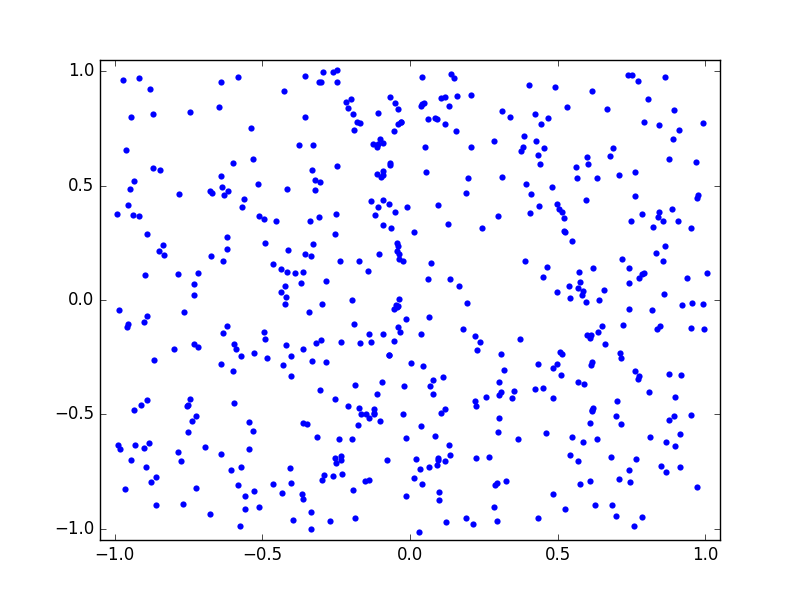}
    \caption{Figure \ref{fig:arbitrary conics first} with low signal-to-noise ratio.}
	\label{fig:arbitrary conics noisy}
	
\end{minipage}

\end{figure}

Figures \ref{fig:two circles noisy} and \ref{fig:arbitrary conics noisy} were obtained by adding background noise to Figures \ref{fig:two circles} and \ref{fig:arbitrary conics first} respectively, in order to obtain low signal-to-noise ratios. Now, it is challenging -- if not impossible -- to visually identify the patterns in the data, even though the patterns do still exist. Our goal is to explore the extent to which we can identify these labels in the presence of background noise.

\subsection{Clustering}

A traditional approach to extract such patterns from the available data is to use \textit{clustering} \citep{CLUSTERING}. In clustering, one tries to find some suitable partition of the dataset in such a way that points belonging to the same set of the partition -- known as cluster -- are similar, for some notion of \textit{similarity}.


Essentially, clustering algorithms try to find a function that maps points in a dataset to an index set for the clusters. In labelling, given a \textit{label}, we want to find points that possess this label. By definition a point can only belong to a single cluster, whereas points may have one label, multiple labels or no labels at all. The goals and techniques required for efficient labelling are different to those used in clustering.

There is some previous work on modifying clustering algorithms to allow data points to belong to multiple clusters. This form of clustering is known as \textit{fuzzy clustering}, where one assigns each data point to a cluster with a membership weight that indicates to which degree the point belongs to the cluster \citep{FUZZY}. For instance, the fuzzy c-means algorithm first proposed in \citep{FCM} is a popular fuzzy clustering algorithm. However, this approach is again separate to the \textit{labelling problem} we propose in this paper, since fuzzy clustering does not deal with situations where some data points genuinely belong to more than one group.

\subsection{Outline of paper}

In Section \ref{sec:labelling}, we introduce the notion of a \textit{label} as a property that might be possessed by points in a cloud. This will be the core object of this paper. After stating some basic properties of labels, in Section \ref{sec:algorithm} we propose and discuss an unsupervised algorithm to identify labels from a given dataset. Then, in Section \ref{sec:false discoveries}, we study what is the probability that a given label is a false discovery. Having some control over the probability of obtaining false discoveries will be important in practical applications. In Section \ref{subsec:overlapping} we apply our algorithm to a synthetic dataset where clustering algorithms will necessarily produce unsatisfactory results, whereas as we will see, our labelling algorithm obtains very visually intuitive labels. Finally, in Section \ref{subsec:pendulum} we showed how simple pendulums hidden in noisy data can be identified in an unsupervised way, following our algorithm.

\section{Labelling}\label{sec:labelling}

Let $\mathcal{X}$ be a topological space, called the \textit{observation space}. We will assume we have a feature map $\Phi:\mathcal{X}\rightarrow F$, with $F$ a topological vector space known as the \textit{feature space}. The triple $(\mathcal{X}, F, \Phi)$ will be called the \textit{augmented observation space}. The feature map generates a vector space of real-valued functions $\mathcal{F}_\Phi:=\{\ell\circ \Phi:\ell\in F'\}$, where $F'$ denotes the dual space of $F$. $\mathcal{F}_\Phi$ will be called the space of potential labels.

We will also assume we are given a probability measure $\mu$ on $\mathcal{X}$, which will be called \textit{background noise} and will act as a reference measure. Given an augmented observation space $(\mathcal{X}, F, \Phi)$ and a background noise $\mu$, we define the \textit{$\mu$-augmented observation space} as the tuple $(\mathcal{X}, F, \Phi, \mu)$.

\begin{example}[Polynomials in several variables]

In the case where the observation space is $\mathbb{R}^d$, we can define an augmented observation space by considering the feature map given by the mapping to a basis of polynomials of fixed degree.

\end{example}

Now, we will introduce a key object in this paper -- the notion of a \textit{label} for a given set of points.

\begin{definition}[Label]\label{def:label}
Let $(\mathcal{X}, F, \Phi, \mu)$ be a $\mu$-augmented observation space, let $\mathcal{C}\subset \mathcal{X}$, and set $0<\delta<1$. A potential label $f\in \mathcal{F}_\Phi$ is called a $(\mu, \delta)$-label for $\mathcal{C}$ if there exists an interval $I\subset \mathbb{R}$ such that $0\in I$, $f(\mathcal{C})\subset I$ and

\begin{equation}\label{eq:condition label}
(f_*(\mu))(I) < \delta,
\end{equation} where $f_*(\mu)$ is the pushforward measure defined as $(f_*(\mu))(I)=\mu(f^{-1}(I))$.
\end{definition}

Intuitively, a potential label for a cloud of points  $\mathcal{C}$ is a linear functional on the feature space $F$, $f\in \mathcal{F}_\Phi$. The potential label $f$ becomes a label if the cloud of points is \textit{unreasonably close} to its zero set.

\begin{definition}

Let $(\mathcal{X}, F, \Phi, \mu)$ be a $\mu$-augmented observation space, and let $\mathcal{C}\subset \mathcal{X}$. Given $0<\delta<1$, we define the set of labels for $\mathcal{C}$ as

$$\mathcal{L}_{\mu, \delta}(\mathcal{C}) := \{f \in \mathcal{F}_\Phi:f\mbox{ is a }(\mu, \delta)\mbox{-label for }\mathcal{C}\}.$$

\end{definition}

Notice that for any $0<\delta_1\leq \delta_2$ and $\mathcal{C}_2\subset \mathcal{C}_1\subset \mathcal{X}$, we have

$$\mathcal{L}_{\mu, \delta_1}(\mathcal{C}_1)\subset \mathcal{L}_{\mu, \delta_2}(\mathcal{C}_2).$$

Now, suppose that $f\in \mathcal{F}_\Phi$ is a $(\mu, \delta)$-label for $\mathcal{C}_1$ and $\mathcal{C}_2$, with $\mathcal{C}_1,\mathcal{C}_2\subset \mathcal{X}$. A natural question one could ask is whether $f$ is a $(\mu, \delta)$-label for $\mathcal{C}_1\cup\mathcal{C}_2$. The answer is negative in general, but we do have the following result.

\begin{proposition}
Let $\mathcal{C}_1,\mathcal{C}_2\subset \mathcal{X}$. Let $f\in \mathcal{F}_\Phi$. If $f$ is a $(\mu,\delta_1)$-label for $\mathcal{C}_1$ and a $(\mu, \delta_2)$-label for $\mathcal{C}_2$, then $f$ is a $(\mu,\delta_1+\delta_2)$-label for $\mathcal{C}_1\cup \mathcal{C}_2$.
\end{proposition}

\begin{proof}
For $i\in \{1,2\}$, since $f$ is a $(\mu, \delta_i)$-label for $\mathcal{C}_i$, there exists an interval $I_i\subset \mathbb{R}$ with $f(\mathcal{C}_i)\cup \{0\} \subset I_i$ such that $(f_*(\mu))(I_i) < \delta_i$. Let $I=I_1\cup I_2$. Then, $f(\mathcal{C}_1\cup \mathcal{C}_2)\cup  \{0\} \subset I$ and $$(f_*(\mu))(I_1\cup I_2) \leq (f_*(\mu))(I_1) + (f_*(\mu))(I_2)< \delta_1+\delta_2$$ and hence $f$ is a $(\mu,\delta_1+\delta_2)$-label for $\mathcal{C}_1\cup \mathcal{C}_2$.
\end{proof}

Therefore, the union of two cloud of points that share a label will also have the same label, but for a different threshold $\delta$.

The following proposition will be useful when it comes to finding labels for a given cloud of points, since it allows us to relax the condition $0\in I$.

\begin{proposition}\label{prop:move labels}
Let $\mathcal{C}\subset \mathcal{X}$, $f\in \mathcal{F}_\Phi$. Assume that there exists an interval $I\subset \mathbb{R}$ with $f(\mathcal{C})\subset I$ such that

$$(f_*(\mu))(I)<\delta$$ for some $\delta>0$. Assume that $\mathcal{F}_\Phi$ contains constants. Then, $f-c$ is a $(\mu,\delta)$-label for $\mathcal{C}$ for all $c\in I$.

\end{proposition}

\begin{proof}
Let $c\in I$, and set $J=I-c$. Then, $0\in J$, and since $\mathcal{F}_\Phi$ contains constants, $f-c\in\mathcal{F}_\Phi$. Now,

$$((f-c)_*(\mu))(J)=(f_*(\mu))(J+c) = (f_*(\mu))(I)<\delta,$$ and hence $f-c$ is a $(\mu,\delta)$-label for $\mathcal{C}$.
\end{proof}

As a corollary, we have the following results about the cardinality of $\mathcal{L}_{\mu,\delta}$.

\begin{corollary}
Let $f$ satisfy the assumptions of the previous proposition. If $I$ contains more than one point, then $|\mathcal{L}_{\mu, \delta}(\mathcal{C})|=\infty$.
\end{corollary}

\begin{proof}
Follows from the fact that, by the previous proposition, $\{f-c\}_{c\in I}$ is an infinite family of $(\mu, \delta)$-labels for $\mathcal{C}$.
\end{proof}

\begin{corollary}\label{corollary:infinite labels}
Let $f$ be a $(\mu,\delta)$-label for $\mathcal{C}$ such that $f(\mathcal{C})\neq \{0\}$. Then, $|\mathcal{L}_{\mu,\delta}(\mathcal{C})|=\infty$.
\end{corollary}

This indicates a big difference between clustering and labelling -- although points are assigned to a single cluster, points will typically have infinite labels since small variations of a label are still labels.

Moreover, notice that $f\in \mathcal{F}_\Phi$ is a $(\mu,\delta)$-label if and only if $\lambda f$ is a $(\mu,\delta)$-label for all $\lambda\in \mathbb{R}\setminus \{0\}$.

However, $\mathcal{L}_{\mu,\delta}(\mathcal{C})$ is not a vector space in general. Indeed, if there exists $f\in \mathcal{L}_{\mu,\delta}(\mathcal{C})$ such that $f(\mathcal{C})\neq \{0\}$ and if $\mathcal{F}_\Phi$ contains constants, then by Proposition \ref{prop:move labels} there exists some $c\in \mathbb{R}$ such that $f-c$ is a $(\mu,\delta)$-label for $\mathcal{C}$. If $\mathcal{L}_{\mu,\delta}(\mathcal{C})$ is a vector space, then $c=f-(f-c)$ would be a $(\mu, \delta)$-label. But this would imply that $1=\mu(\mathcal{X}) < \delta$, which is a contradiction.

\section{Extracting labels from data}\label{sec:algorithm}

In practice, we are given a dataset $\mathcal{C}\subset \mathcal{X}$, and we want to find labels for points contained in $\mathcal{A}$. In othe words, we are interested in extracting labels of subsets of $\mathcal{A}$. Therefore, we are interested in finding

$$\mathbb{L}_{\mu, \delta}(\mathcal{A}) = \{ (\mathcal{C},  \mathcal{L}_{\mu, \delta}(\mathcal{C}) ) : \mathcal{C}\subset \mathcal{A}\}.$$

In this section we propose an algorithm to estimate $\mathbb{L}_{\mu, \delta}(\mathcal{A})$ for a given cloud of points $\mathcal{A}\subset \mathcal{X}$ and $\delta>0$, in the case where the feature space $F$ is a $D$-dimensional Hilbert space $(F, (\cdot,\cdot)_F)$. More specifically, the algorithm will find labels for cloud of points in $\mathcal{A}$ containing more than $n_0>0$ points, where $n_0$ is an input of the algorithm that acts as a minimum threshold for the labels.

To achieve this task, we introduce two algorithms -- an algorithm to estimate $\mathcal{L}_{\mu, \delta}(\mathcal{C})$ for some $\mathcal{C}\subset \mathcal{A}\subset \mathcal{X}$, and an algorithm to estimate $\mathbb{L}_{\mu,\delta}(\mathcal{A})$, which will use the first one. Standard and well-known tools -- KPCA and SVD -- are used in the algorithms.

\subsection{Estimating $\mathcal{L}_{\mu,\delta}$}\label{subsec:labels of set}

Let $\mathcal{C}\subset \mathcal{A}\subset \mathcal{X}$. We will try to obtain an estimation of $\mathcal{L}_{\mu, \delta}(\mathcal{C})$. Since by Corollary \ref{corollary:infinite labels} we know that $\mathcal{C}$ will typically have infinite labels, we will try to obtain a candidate for a $(\mu,\delta)$-label for $\mathcal{C}$. Then, we check if this candidate is indeed a $(\mu, \delta)$-label for $\mathcal{C}$, and we will either return the empty set or this candidate as an estimate of $\mathcal{L}_{\mu,\delta}(\mathcal{C})$.

\begin{algorithm}[t]
	\SetAlgoLined
	\SetKwInOut{Parameters}{Parameters}
	\Parameters           {$\mathcal{C}\subset \mathcal{X}$: Finite set to be labelled  \\
		$\Phi:\mathcal{X}\rightarrow F$: Feature map, where $F$ is a Hilbert space with inner product $(\cdot,\cdot)_F$\\
		$\mu$: Background noise \\
		$\delta>0$: Label probability threshold \\
		 }
	\KwOut{Estimation of $\mathcal{L}_{\mu,\delta}(\mathcal{C})$.}
	\BlankLine
	
	Obtain the right singular vector corresponding to the least singular value of the linear operator $v\in F' \mapsto (\langle v, \Phi(x)\rangle)_{i=1}^{|\mathcal{C}|} \in \mathbb{R}^N$ using SVD\;
	Set $f\leftarrow \ell \circ \Phi \in \mathcal{F}_\Phi$\;
	Set $I\subset \mathbb{R}$ to be the smallest closed interval such that $f(\mathcal{C})\subset I$\;
	Estimate $(f_*(\mu))(I)$\;
	\uIf{$(f_*(\mu))(I)<\delta$}{
		\uIf{$0\in I$} {
			$\mathcal{L}_{\mu, \delta}(\mathcal{C})\leftarrow \{f\}$\;
		}
		\uElseIf{$\mathcal{F}_\Phi$ contains constants} {
			$I_\star\leftarrow $ midpoint of $I$\;
			$\mathcal{L}_{\mu, \delta}(\mathcal{C})\leftarrow \{f-I_\star\}$\;
		}
		\Else {
			$\mathcal{L}_{\mu, \delta}(\mathcal{C})\leftarrow \varnothing$\;
		}
		
	}
	\Else {
		$\mathcal{L}_{\mu, \delta}(\mathcal{C})\leftarrow \varnothing$\;
	}
	\Return{$\mathcal{L}_{\mu, \delta}(\mathcal{C})$}
	\caption{Algorithm to estimate $\mathcal{L}_{\mu,\delta}(\mathcal{C})$.}
	\label{algo:labels}
\end{algorithm}

This candidate will be obtained assuming that the feature space $F$ is a $D$-dimensional Hilbert space with an inner product $(\cdot,\cdot)_F$, in order to then apply Kernel Principal Component Analysis (KPCA). KPCA was first introduced in \citep{KPCA}, and it has since been applied in many situations where there exists some \textit{nonlinear pattern} in the data, such as in certain classification problems \citep{APPLICATION-KPCA1, APPLICATION-KPCA2} or clustering problems \citep{APPLICATION-KPCA3, APPLICATION-KPCA4}.

Kernel Principal Component Analysis is a natural nonlinear extension of Principal Component Analysis \citep{PCA}. Essentially, it begins by mapping the dataset to a feature space via a given feature map, in order to then apply PCA.


The idea is to extract the first $D-1$ kernel principal components of $\mathcal{A}$, $\{v^k\}_{k=1}^{D-1}$. Then, $\left (\mbox{span}\{v^k\}_{k=1}^{D-1}\right )^\perp$ has dimension 1, so we can extract $0\neq \ell \in \left (\mbox{span}\{v^k\}_{k=1}^{D-1}\right )^\perp$. Then, $f=\ell\circ\Phi\in \mathcal{F}_\Phi$ will be our candidate for $(\mu,\delta)$-label for $\mathcal{C}$.

Let $L = \dfrac{1}{|\mathcal{C}|} \sum_{x\in \mathcal{C}}\Phi(x)\otimes \Phi(x)$. Following \citep{KPCA}, the first $D-1$ kernel principal components of $\mathcal{C}$ are given by the first $D-1$ eigenvectors of $L$. Since $L$ is symmetric its eigenvectors are orthogonal, and therefore we may pick $0\neq \ell\in \left (\mbox{span}\{v^k\}_{k=1}^{D-1}\right )^\perp$ to be the least eigenvector of $L$. However, the least eigenvector of $L$ is given by the least right singular vector of the linear operator $v\in F' \mapsto (\langle v, \Phi(x)\rangle)_{i=1}^{|\mathcal{C}|} \in \mathbb{R}^N$ which can be found using SVD. Notice that in the case where $F=\mathbb{R}^D$, this linear operator is just a matrix in $\mathbb{R}^{|\mathcal{C}|\times D}$.

This procedure is shown in Algorithm \ref{algo:labels}, where Proposition \ref{prop:move labels} was also taken into account. We are not claiming that this algorithm can find all the labels for a dataset. We are claiming using the standard SVD provides a robust methodology to find these labels.

\subsection{Estimating $\mathbb{L}_{\mu,\delta}$}

Now, the objective will be to estimate $\mathbb{L}_{\mu,\delta}(\mathcal{A})$ for a given $\mathcal{A}\subset \mathcal{X}$. However, when we label subsets $\mathcal{C}\subset \mathcal{A}$, we will only consider subsets with a certain minimum number of elements, i.e. $|\mathcal{C}|\geq n_0 \in \mathbb{N}$. Exploiting the fact that $\mathcal{L}_{\mu, \delta}(\mathcal{A}_1)\subset \mathcal{L}_{\mu,\delta}(\mathcal{A}_2)$ whenever $\mathcal{A}_2\subset \mathcal{A}_1$ and using Algorithm \ref{algo:labels} to estimate $\mathcal{L}_{\mu,\delta}$, we propose Algorithm \ref{algo:labelling}.

\begin{algorithm}[t]
	\SetAlgoLined
	\SetKwInOut{Parameters}{Parameters}
	\Parameters           {$\mathcal{A}\subset \mathcal{X}$: Finite set to be labelled  \\
		$\Phi:\mathcal{X}\rightarrow F$: Feature map, where $F$ is a Hilbert space with inner product $(\cdot,\cdot)_F$\\
		$\mu$: Background noise \\
		$\delta>0$: Label probability threshold \\
		$n_0 \in \{1, 2, \ldots, |\mathcal{A}|\}$: Minimum size of labelled points.\\
		$N\in \mathbb{N}$: Number of steps.
		 }
	\KwOut{An estimation of $\mathbb{L}_{\mu, \delta}(\mathcal{A})$.}
	\BlankLine
	initialise $\mathbb{L}_{\mu, \delta}(\mathcal{A})\leftarrow \varnothing$\;

	\For{{\upshape each	}$i\in \{1, 2, \ldots, N\}$ }{
		choose randomly a subset $\mathcal{C}\subset \mathcal{A}$ of size $n_0$\;

		estimate $L\leftarrow \mathcal{L}_{\mu, \delta}(\mathcal{C})$\;
		\If{$L\neq \varnothing$}{
			\For{{\upshape each } $p\in \mathcal{A}\setminus {\mathcal{C}}$}{
				estimate $L^\ast \leftarrow \mathcal{L}_{\mu, \delta}(\mathcal{C}\cup \{p\})$\;
				\If{$L^\ast \neq \varnothing$}{
					$\mathcal{C}\leftarrow \mathcal{C}\cup \{p\}$\;
					$L \leftarrow L^\ast$\;
				}
			}
		\If{$\not\exists (\mathcal{D}, \mathcal{L})\in \mathbb{L}_{\mu, \delta}(\mathcal{A})$ {\upshape such that }$\mathcal{D}=\mathcal{C}$}{
			$\mathbb{L}_{\mu, \delta}(\mathcal{A}) \leftarrow \mathbb{L}_{\mu, \delta}(\mathcal{A}) \cup \{(\mathcal{C}, L)\}$\;
		}
		}

	}
	\Return{$\mathbb{L}_{\mu, \delta}(\mathcal{A})$}
	\caption{Algorithm to estimate $\mathbb{L}_{\mu,\delta}(\mathcal{A})$.}
	\label{algo:labelling}
\end{algorithm}

\section{Analysis of false discoveries}\label{sec:false discoveries}

It is important to analyse what is the probability of finding false discoveries. In other words, given some $\delta>0$, what is the probability that there exists at least a $(\mu,\delta)$-label? How does this depend on the value of $\delta$ and the sample size? In what sense does this depend on the feature map we select? How can we pick an appropriate $\delta$ to avoid having false discoveries with high probability?

In this section, we will study and answer these questions. We will do so by analysing the spectral properties of some random matrices. More specifically, we will be interested in the behaviour of the least singular value of some random matrices on the feature space. As we will see, obtaining lower bounds for the least singular value will allow us to answer the questions stated above, using the fact that the least singular value of a linear compact operator $\Gamma$ is characterised as $$\smin(\Gamma) = \inf _{ \lVert x \rVert = 1} \lVert \Gamma x \rVert.$$

Assume that the feature space is a $D$-dimensional Hilbert space with inner product $(\cdot,\cdot)_F$. Let $X_1, \ldots, X_N$ be $N$ points independently sampled according to $\mu$. Define the linear compact operator

\begin{align*}
\Gamma_\Phi : F' &\rightarrow \mathbb{R}^N\\
\ell & \mapsto (\langle \ell, \Phi(X_i)\rangle)_{i=1}^N.
\end{align*} In this section, we will study the least singular value of $\Gamma_\Phi$ since, as we will see, being able to control the least singular value will give us valuable insights on the study of the problem of false discoveries.

A lot of work has been done in the task of analysing the least singular values of rectangular random matrices. For instance, in \citep{BAI}, the authors prove that under some assumptions, $s_{\text{min}}(\Gamma_\Phi)/\sqrt{N}$ converges to some positive limit, where $s_\text{min}(\Gamma_\Phi)$ denotes the least singular value of $\Gamma_\Phi$. In \citep{RECTANGULARSINGULAR}, on the other hand, the authors prove a non-asymptotic estimate of the least singular value of a rectangular random matrix. In both cases -- and in most of the published results -- a big assumption that is made is that the entries of the random matrix are independent, which is not applicable in our case. However, $X_1,\ldots,X_N$ were independently sampled, so the objective is to exploit this feature to extract useful estimates for the least singular value of $\Gamma_\Phi$.

\subsection{Singular values of a class of linear operators}

Let $X$ be a random variable in $F$. Let $X_1, \ldots, X_N$ be $N$ independent copies of $X$, and define the linear compact operator

\begin{align*}
\Gamma : F' &\rightarrow \mathbb{R}^N\\
\ell & \mapsto (\langle \ell, X_i\rangle)_{i=1}^N.
\end{align*}

We will now study the smallest singular value of $\Gamma$, which will be denoted by $\smin(\Gamma)$. However, we will have to make some assumptions on the random variable $X$. A crucial assumption will be related to the so-called Orlicz norm $\lVert X \rVert_{\psi_\alpha}$,where $\alpha\geq 1$. It is defined as

$$\lVert X \rVert_{\psi_\alpha} = \inf \left \{C>0:\mathbb{E}\left (\exp \left (\dfrac{\lVert X \rVert_F^\alpha}{C^\alpha}\right )\right )\leq 2\right  \}.$$

This norm controls the tail of $\lVert X \rVert_F$, where $\lVert \cdot\rVert_F$ denotes the norm in $F$, since one can show that if $\lVert X \rVert_{\psi_\alpha} < \infty$, then $\mathbb{P}(\lVert X \rVert_F>\lambda) \leq 2\exp (-\lambda^\alpha / \lVert X \rVert_{\psi_\alpha}^\alpha)$.

Examples of such random variables include Gaussian random variables for $\alpha=2$, and bounded random variables, which have finite $\psi_\alpha$-norms for all $\alpha\geq 1$.

In this section, the following lemma from \citep{BOOKVERSHYNIN} will be useful:

\begin{lemma}[{{\cite[Lemma 5.36]{BOOKVERSHYNIN}}}]\label{lemma:smin}
Let $A:F'\rightarrow \mathbb{R}^N$ be a linear operator such that

$$\lVert A^\ast A - I\rVert\leq \delta \vee \delta^2$$ for some $\delta>0$, where the norm is the operator norm. Then,

$$1-\delta \leq \smin(A) \leq 1+\delta.$$
\end{lemma}

A key result that will be used in this section is the following theorem, first proved in \citep{INDEPENDENTROWS}.

\begin{theorem}[{{\cite[Theorem 2.1]{INDEPENDENTROWS}}}]\label{th:independent rows}
Let $X$ be a random variable in $F$ such that:

\begin{enumerate}
\item There exists some $\rho>0$ such that $(\mathbb{E}[|\langle \theta, X\rangle |^4)^{1/4}\leq \rho$ for all $\theta\in F'$ with $\lVert \theta \rVert_{F'}=1$.
\item $\lVert X \rVert_{\psi_\alpha < \infty}$ for some $\alpha\geq 1$.
\end{enumerate}

Let $\Lambda := \mathbb{E} (X^\ast X)$. Set $p_{D, N} = \left ( \frac{1}{B_{D, N} \vee A_{D,N}^2}\right )^\beta$, where

$$A_{D, N} := \lVert X \rVert_{\psi_\alpha}\dfrac{\sqrt{\log D}(\log N)^{1/\alpha}}{\sqrt{N}},\quad B_{D,N}:=\dfrac{\rho^2}{\sqrt{N}}+\lVert \Lambda\rVert^{1/2} A_{D, N},$$ and $\beta=(1+2/\alpha)^{-1}$.

Then, there exists a constant $c>0$ such that for all $\delta\geq 0$,

$$\mathbb{P}\left (\left\lVert \dfrac{1}{N} \Gamma^\ast \Gamma - \Lambda\right \rVert \geq \delta \right )\leq \exp (-c\delta^\beta p_{D,N}).$$
\end{theorem}

Joining these two results, we can prove the following theorem:

\begin{theorem}\label{th:estimate smin}
Let $X$ be a random variable taking values in $F$ that satisfies the following conditions:

\begin{enumerate}
\item $X$ is isotropic, that is, $\mathbb{E}(X^\ast X)=I$.
\item $\lVert X \rVert_{\psi_\alpha} < \infty$ for some $\alpha\geq 1$.
\end{enumerate}

Let $\gamma\leq 1/2$. Set, for $t>0$,

$$p_{D, N, t, \gamma} = t^\beta N^{\beta(\gamma-1/2)} p_{D,N}.$$ Then, with probability at least $1-\exp(-cp_{D, N, t, \gamma})$,

\begin{equation}\label{eq:estimate smin}
\sqrt{N} - tN^\gamma \leq \smin(\Gamma) \leq \sqrt{N} + tN^\gamma.
\end{equation}
\end{theorem}

\begin{proof}
The assumptions on $X$ imply that the assumptions from Theorem \ref{th:independent rows} are satisfied as well, and that $\Lambda = I$. Therefore, applying Theorem \ref{th:independent rows} with $\delta = tN^{\gamma - 1/2}$, we have

$$\left \lVert \dfrac{1}{N} \Gamma^\ast \Gamma - I \right \rVert \leq tN^{\gamma-1/2}$$ with probability at least $1-\exp(-c t^\beta N^{\beta(\gamma-1/2)} p_{D,N})=1-\exp(-cp_{D,N,t,\gamma})$.

Therefore, applying Lemma \ref{lemma:smin} with $A=\frac{1}{\sqrt{N}} \Gamma$, we conclude that

$$\sqrt{N} - tN^\gamma \leq \smin(\Gamma) \leq \sqrt{N} + tN^\gamma$$ with probability higher than $1-\exp(-cp_{D,N,t,\gamma})$.
\end{proof}

\begin{remark}\label{remark:pdnt}
The term $p_{D,N,t,\gamma}$ grows faster than $t^\beta N^{r\beta}$ for all $r<\gamma$. Thus, there is a trade-off between how fast $p_{D,N,t,\gamma}$ grows, and how good the estimate \eqref{eq:estimate smin} is.
\end{remark}

We may now move to the operator $\Gamma_\Phi$, and deduce a result similar to Theorem \ref{th:estimate smin}. But first, we will state and proof the following lemma.

\begin{lemma}\label{lemma: orlicz}
Let $\mathcal{X}$ be a Banach space. Let $\Phi:\mathcal{X} \rightarrow F$ be a feature map that satisfies the growth condition $\lVert \Phi(x) \rVert_F \leq C_0(1+\lVert x \rVert_{\mathcal{X}}^k)$ for all $x\in \mathcal{X}$, where $C_0>0$ is a constant and $k>0$. Let $X$ be a random variable taking values in $\mathcal{X}$, and $\alpha\geq k$ such that $\lVert X \rVert_{\psi_\alpha} < \infty$. Then, $\lVert \Phi(X) \rVert_{\psi_{\alpha/k}}<\infty$.
\end{lemma}

\begin{proof}
Notice that since $\alpha\geq k$, $\alpha/k\geq 1$. Now, given $C>0$,

$$\mathbb{E} \left (\exp \left (\dfrac{\lVert \Phi(X) \rVert^{\alpha/k}}{C}\right )\right )\lesssim \mathbb{E}\left (\exp \left (\dfrac{1 + \lVert X \rVert^\alpha}{C/C_0}\right )\right )$$ and therefore $\lVert \Phi(X)\rVert_{\psi_{\alpha/k}}<\infty.$
\end{proof}

We may now state and proof the following theorem about the least singular value of $\Gamma_\Phi$:

\begin{theorem}\label{th:estimate smin feature map}
Let $X$ be a random variable in a Banach space $\mathcal{X}$, such that $X$ is distributed according to $\mu$. Let $\Phi:\mathcal{X}\rightarrow F$ be a feature map. Assume that

\begin{enumerate}
\item $\Phi(X)$ is isotropic, that's it, $\mathbb{E}_\mu (\Phi(X)^\ast \Phi(X)) = I$.
\item $\Phi$ satisfies the growth condition $\lVert \Phi(x)\rVert_F \lesssim 1+\lVert x \rVert_\mathcal{X}^k$ for all $x\in \mathcal{X}$, where $k>0$.
\item $\lVert X \rVert_{\psi_\alpha}<\infty$ for some $\alpha\geq k$.
\end{enumerate}

Let $X_1, \ldots, X_N$ be $N$ independent copies of $X$. Then, if

\begin{align*}
\Gamma_\Phi : F' &\rightarrow \mathbb{R}^N\\
\ell & \mapsto (\langle \ell, \Phi(X_i)\rangle)_{i=1}^N,
\end{align*} given any $t>0$ and $\gamma \leq 1/2$ we have

$$\sqrt{N} - tN^\gamma \leq \smin(\Gamma_\Phi) \leq \sqrt{N} + tN^\gamma$$ with probability higher than $1-\exp (-cp_{D, N, t, \gamma})$.

\end{theorem}

\begin{proof}
By Lemma \ref{lemma: orlicz}, and the assumptions of the theorem, all the assumptions of Theorem \ref{th:estimate smin} are satisfied. Therefore, we may apply Theorem \ref{th:estimate smin} and the result follows.
\end{proof}

\begin{remark}\label{remark:white}
One may think that the assumption that $\Phi(X)$ is isotropic is quite strong. However, this assumption is actually very weak. Indeed, under the mild assumption that the correlation matrix $\Lambda = \mathbb{E}_\mu (\Phi(X)^\ast \Phi(X))$ is invertible, then $\Lambda^{-1/2}\Phi(X)$ is isotropic. We may define now a new feature map using the whitening transformation

\begin{align*}
\widetilde{\Phi}:\mathcal{X} &\rightarrow F\\
x&\mapsto \Lambda^{-1/2}\Phi(x).
\end{align*}

The space of potential labels doesn't change when this transformation is applied (i.e. $\mathcal{F}_{\widetilde{\Phi}} = \mathcal{F}_{\Phi}$), but since $\widetilde{\Phi}(X)$ is now isotropic, we may apply Theorem \ref{th:estimate smin feature map}.

\end{remark}

Therefore, this theorem implies that $\smin(\Gamma_\Phi)\sim \sqrt{N}$. This behaviour can be seen in Figure \ref{fig:smin_over_sqrtn}, where we simulate $\smin(\Gamma_\Phi)/\sqrt{N}$ for different values of $N$, the feature map \linebreak $\Phi:\mathbb{R}^2\rightarrow \mathbb{R}^6$ given by the whitening transformation from Remark \ref{remark:white} applied to the quadratic monomials, and a background noise $\mu$ given by the uniform distribution on $[-1,1]^2$. To compare, we plotted $\smin(\Gamma_\Phi)/ \sqrt{N}$ for clouds of points sampled from a circle with different levels of noise. As we see, the numerical tests back up the theoretical results from Theorem \ref{th:estimate smin feature map}, and there seems to be a gap between the dataset of pure noise, and the datasets with signal.

\subsection{Probability of a false discovery}

We will now address the problem of estimating the probability that a potential label is a label for a cloud of points independently sampled from the background noise $\mu$. We will then use this result to show that the probability of obtaining false discoveries decays exponentially with the size of the cloud of points.

The following theorem offers an insight on what is the probability that a potential label is a false discovery.

\begin{theorem}\label{th: delta}
Let $(\mathcal{X}, F, \Phi, \mu)$ be a $\mu$-augmented observation space. Let $\mathcal{C}\subset \mathcal{X}$ be a set of $N$ points independently sampled according to the background noise $\mu$. Suppose that the assumptions of Theorem \ref{th:estimate smin feature map} hold. Let $f=\ell\circ \Phi\in \mathcal{F}_\Phi$ be a potential label, where $\ell\in F'$ with $\lVert \ell \rVert_{F'}=1$. Set

$$\delta_f = (f_*(\mu))([0, 1-tN^{\gamma-1/2}])\wedge (f_*(\mu))([-1+tN^{\gamma-1/2}, 0]),$$ with $t>0$ and $\gamma\leq 1/2$. Then, with probability higher than $1-\exp(-cp_{D,N,t,\gamma})$, $f$ is not a $(\mu, \delta)$-label for any $0<\delta\leq \delta_f$.
\end{theorem}

\begin{proof}
Let $0<\delta\leq \delta_f$. We have

$$\lVert \Gamma_\Phi \ell \rVert_\infty \geq \dfrac{1}{\sqrt{N}} \lVert \Gamma_\Phi \ell \rVert_2 \geq \dfrac{1}{\sqrt{N}}\smin(\Gamma_\Phi).$$

On the other hand, given $I\subset \mathbb{R}$ such that $f(\mathcal{C})\cup \{0\}\subset I$, we have

$$[0, \lVert \Gamma_\Phi \ell \rVert_\infty]\subset I\quad\mbox{ or }\quad [-\lVert \Gamma_\Phi \ell\rVert_\infty, 0]\subset I.$$

Therefore, by Theorem \ref{th:estimate smin feature map},

\begin{align*}
(f_*(\mu))(I) &\geq (f_*(\mu))\left (\left [0, \frac{1}{\sqrt{N}}\smin(\Gamma_\Phi)\right ]\right )\wedge (f_*(\mu))\left (\left [-\frac{1}{\sqrt{N}}\smin(\Gamma_\Phi), 0\right ]\right )\\
&\geq (f_*(\mu))([0, 1-tN^{\gamma-1/2}])\wedge (f_*(\mu))([-1+tN^{\gamma-1/2}, 0]) = \delta_f\geq \delta
\end{align*} with probability higher than $1-\exp(-cp_{D, N, t, \gamma})$, as desired.

\end{proof}

Now, as a consequence of Theorem \ref{th: delta} we have the following result, which ensures that the probability of obtaining a false discovery from a pure-noise dataset decays exponentially fast with the number of points in the dataset.

\begin{lemma}[{{\cite[Lemma 2.3.4]{TAO}}}]\label{lemma: cardinality}
Let $(\mathcal{X}, F, \Phi, \mu)$ be a $\mu$-augmented observation space and let $\Sigma_\epsilon\subset F'$ be a maximal $\epsilon$-net of $F'$, i.e. a subset of $F'$ whose points are separated from each other by a distance of at least $\epsilon$ and is maximal with respect to the inclusion order. Then, $|\Sigma_\epsilon|<(C/\epsilon)^D$, with $C>0$ an absolute constant.
\end{lemma}

\begin{theorem}\label{th: false discovery}
Let $(\mathcal{X}, F, \Phi, \mu)$ be a $\mu$-augmented observation space, and let $\mathcal{C}\subset \mathcal{X}$ be a set of $N$ points independently sampled according to the background noise $\mu$. Assuming the assumptions of Theorem \ref{th:estimate smin feature map} hold, for any $t, \epsilon>0,\gamma\leq 1/2$ there exists $\delta_{\epsilon, t, \gamma,N}>0$ such that the probability that there exists a $(\mu, \delta_{\epsilon, t, \gamma,N})$-label for $\mathcal{C}$ is less than $(C/\epsilon)^D e^{-p_{D, N, t, \gamma}}$, with $C>0$ a constant.

\end{theorem}

\begin{proof}
Set $\epsilon>0$, and let $\Sigma_\epsilon$ be an $\epsilon$-net of $F'$. Let $\delta>0$ and suppose that $f=\ell\circ \Phi$ with $\lVert \ell\rVert_{F'}=1$ is a $\delta$-label. There exists some $\ell_\epsilon\in \Sigma_\epsilon$ such that $\lVert \ell-\ell_\epsilon\rVert_{F'}<\epsilon$. Set $f_\epsilon=\ell_\epsilon\circ \Phi$. Now, given an interval $I\subset \mathbb{R}$ with $\{0\}\cup f( \mathcal{C})\subset I$, and recalling that by Lemma \ref{lemma: orlicz} the feature map $\Phi$ is bounded by $N$ with probability higher than $1-2\exp(-C_0N^{\alpha/k})$ we have, with high probability,

$$f_\epsilon(x) = f(x) + ((\ell_\epsilon - \ell)\circ \Phi)(x) \in I_{\epsilon N}\quad \forall x\in \mathcal{C},$$ and thus $\{0\} \cup f_\epsilon(\mathcal{C}) \subset I_{\epsilon N}$, where $I_{\eta} := \{x\in \mathbb{R}:\mbox{dist}(x, I) < \eta\}$. Now,

\begin{align*}
&\mathbb{P}^\mu(f_\epsilon \in I_{\epsilon N}) = \mathbb{P}^\mu (f + (\ell_\epsilon - \ell)\circ \Phi\in I_{\epsilon N}) \leq \mathbb{P}^\mu(f\in I_{2\epsilon N}) \\
&= \mathbb{P}^\mu (f\in I) + \mathbb{P}^\mu (f\in I_{2\epsilon N}\setminus I) < \delta + c_{\epsilon, N}
\end{align*} where $c_{\epsilon, N}$ depends on $\epsilon, N$ and the $\mu$-augmented observation space. Thus, setting $\delta_{\epsilon,t,\gamma,N}:=\min_{\ell \in \Sigma_\epsilon} \delta_{\ell\circ \Phi} - c_{\epsilon,N}$ where $\delta_{\ell \circ \Phi}$ was defined in Theorem \ref{th: delta},

\begin{align*}
&\mathbb{P}^\mu \left (\bigvee_{\substack{\ell\in F'\\ \lVert \ell\rVert_{F'}=1}} \ell\circ \Phi\mbox{ is a }(\mu, \delta_{\epsilon,t,\gamma,N})\mbox{-label for }\mathcal{C}\right ) \leq \mathbb{P}^\mu \left (\bigvee_{\ell\in \Sigma_\epsilon} \ell\circ \Phi\mbox{ is a }(\mu, \delta_{\epsilon,t,\gamma,N} + c_{\epsilon,N})\mbox{-label for }\mathcal{C}\right )\\
&\leq \sum_{\ell\in \Sigma_\epsilon} \mathbb{P}^\mu (\ell\circ \Phi\mbox{ is a }(\mu, \delta_{\epsilon,t,\gamma,N} + c_{\epsilon,N})\mbox{-label for }\mathcal{C}) \leq  \sum_{\ell\in \Sigma_\epsilon} \mathbb{P}^\mu (\ell\circ \Phi\mbox{ is a }(\mu, \delta_{\ell \circ \Phi})\mbox{-label for }\mathcal{C}) \\
&\leq (C/\epsilon)^D e^{-p_{D, N, t, \gamma}}
\end{align*} where Lemma \ref{lemma: cardinality} and Theorem \ref{th: delta} were used in the last step.
\end{proof}

By Remark \ref{remark:pdnt}, $p_{D, N, t, \gamma}\rightarrow \infty$ as $N\rightarrow \infty$, provided that $\gamma\geq 0$. Also, it is easy to check that $p_{D, N, t, \gamma}\rightarrow 0$ as $D\rightarrow \infty$. Hence, as the size of the cloud of points increases the probability of having false discoveries becomes increasingly smaller, as one may have suspected. However, the more complex the feature map is, that's it, the higher the dimension of the feature space is, the higher the probability of obtaining a false discovery is. This also makes sense, since if the dimension of the feature space is high compared to the number of points, it is easier to find a hyperplane that is \textit{close} to the cloud of points.

Figure \ref{fig:delta0} shows $\delta_f$ plotted as a function of $N$. As we see, $\delta_f$ increases rapidly. Moreover, given the exponential nature of $1-\exp(-cp_{D,N,t,\gamma})$, even with a relatively small sample of points, the probability of obtaining a false discovery is very low.

\begin{figure}
\centering
\begin{minipage}[t]{.45\textwidth}
\centering
   \includegraphics[width=\textwidth]{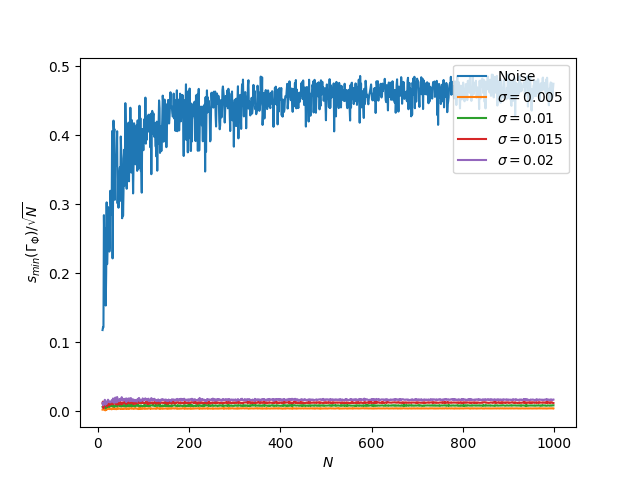}
    \caption{$\smin(\Gamma_\Phi)/\sqrt{N}$ for the feature map $\Phi:\mathbb{R}^2\rightarrow \mathbb{R}^6$ given by $\Phi((x_1,x_2))=(1,x_1,x_2,x_1^2,x_2^2,x_1x_2)$. The blue plot corresponds to a cloud of points sampled from a uniform distribution on $[-1, 1]^2$, and the other plots correspond to a circle with different levels of noise.}
	\label{fig:smin_over_sqrtn}
\end{minipage}\hfill
\begin{minipage}[t]{.45\textwidth}
	\centering
   \includegraphics[width=\textwidth]{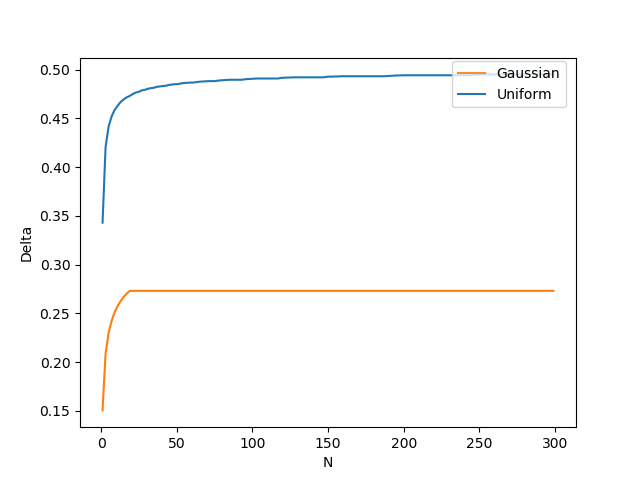}
    \caption{$\delta_f$ as a function of $N$, for $\gamma=0.3$, $t=0.7$ and potential label $f:\mathbb{R}^2\rightarrow \mathbb{R}$ given by $f((x_1,x_2)) = x_1^2 + x_2 ^2 - 0.8^2$. Two background noise measures $\mu$ were used: a uniform distribution on $[-1, 1]^2$, and a $\mathcal{N}(0, I)$ normal distribution. As we see, $\delta_f$ rapidly converges as $N$ increases.}
	\label{fig:delta0}

\end{minipage}

\end{figure}

\section{Numerical experiments}\label{sec:experiments}
\subsection{Overlapping conics}\label{subsec:overlapping}
We set $\mathcal{X}=[-1,1]^2$, $F=\mathbb{R}^6$ and the feature map $\Phi:[-1,1]^2\rightarrow \mathbb{R}^6$ as $\Phi((x_1,x_2)) = (1,x_1,x_2,x_1^2, x_2^2, x_1x_2)$, which allows us to identify conic labels. The background noise measure $\mu$ we picked was the uniform distribution on $\mathcal{X}$.

We tested Algorithm \ref{algo:labelling} with the synthetic dataset shown in Figure \ref{fig:two circles}. The cloud of points $\mathcal{C}\subset \mathcal{X}$ shown in Figure \ref{fig:two circles} was obtained randomly sampling from two circles, and then adding noise with a $\mathcal{N}(0, \sigma^2)$ distribution with $\sigma=0.02$. The cloud of points has 200 points -- 100 points from each circle.

The reason why this dataset was picked is that it is a visually simple cloud of points, but one where clustering would necessarily produce undesired results since it will partition the data into disjoint clusters, which is not very intuitive in this dataset. We also tested the algorithm for the case where we add background noise to the cloud of points in Figure \ref{fig:two circles}.

In the experiment we were interested in finding $(\mu,\delta)$-labels with $\delta=0.05$.

\begin{figure}[h]
  \centering
   \includegraphics[width=\textwidth]{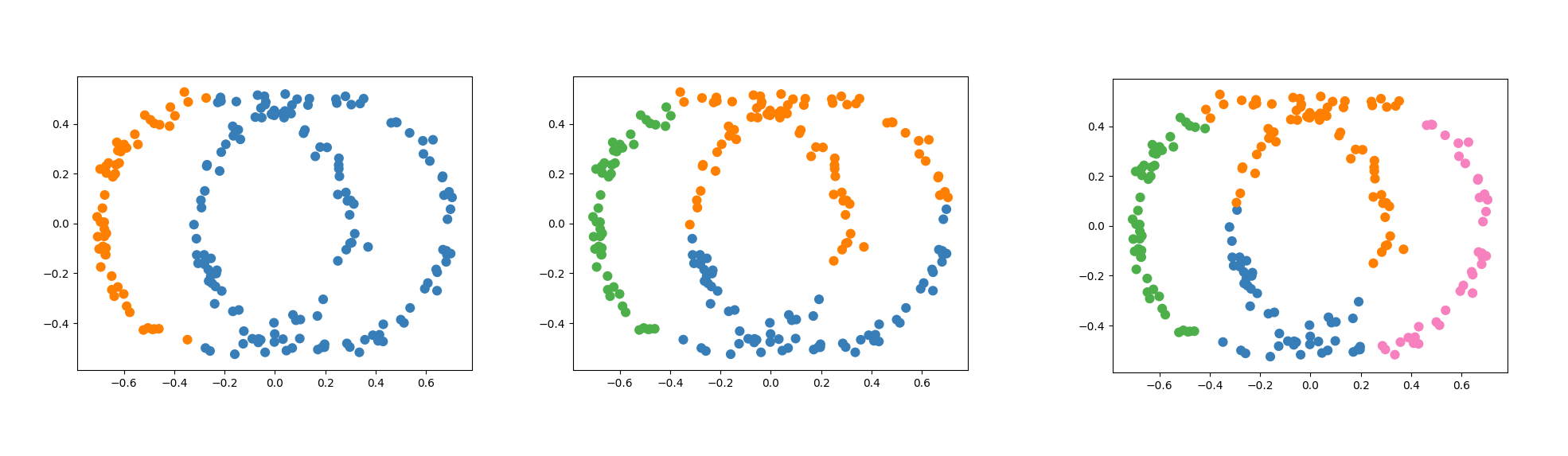}
    \caption{Clusters found by the spectral clustering algorithm proposed in \citep{SPECTRAL}, when trying to find 2, 3 and 4 clusters.}
	\label{fig:spectral}

\end{figure}

\subsubsection{Without adding background noise}

Figure \ref{fig:labelled_overlapping} shows the results obtained with Algorithm \ref{algo:labelling}. As we see, the algorithm is able to successfully identify the two circles. It also captures two other labels: an outer and inner ellipse, which also makes sense.

\begin{figure}[h]
  \centering
  \begin{minipage}[b]{0.49\textwidth} 
    \includegraphics[width=0.8\textwidth]{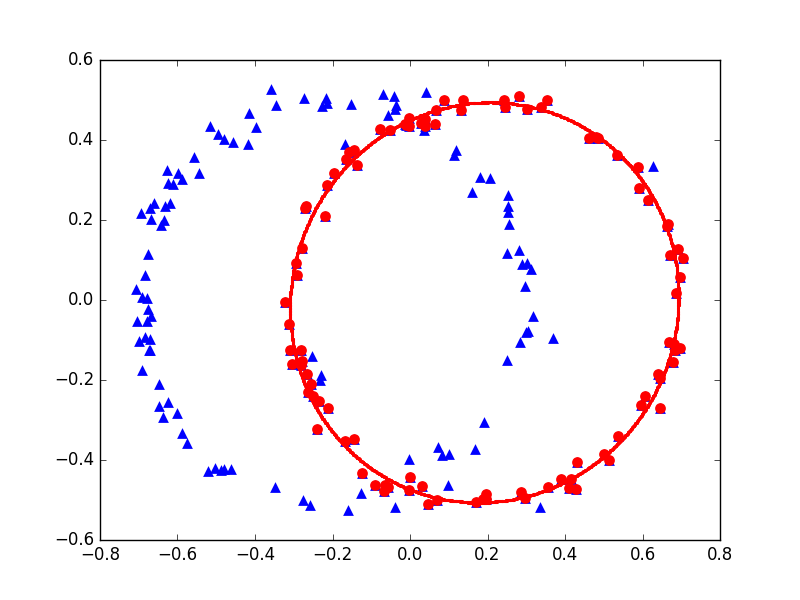}

  \end{minipage}
  \hfill
  \begin{minipage}[b]{0.49\textwidth}
    \includegraphics[width=0.8\textwidth]{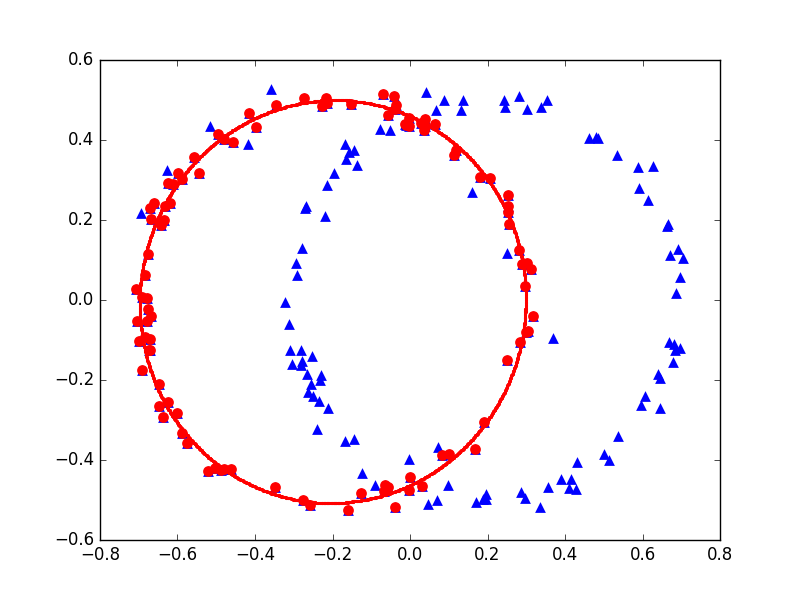}
  \end{minipage}
  \begin{minipage}[b]{0.49\textwidth} 
    \includegraphics[width=0.8\textwidth]{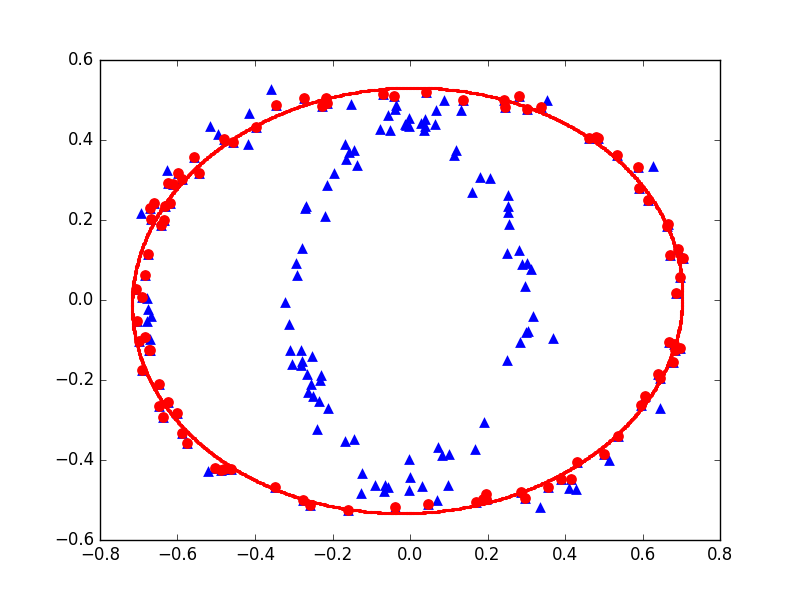}

  \end{minipage}
  \hfill
  \begin{minipage}[b]{0.49\textwidth}
    \includegraphics[width=0.8\textwidth]{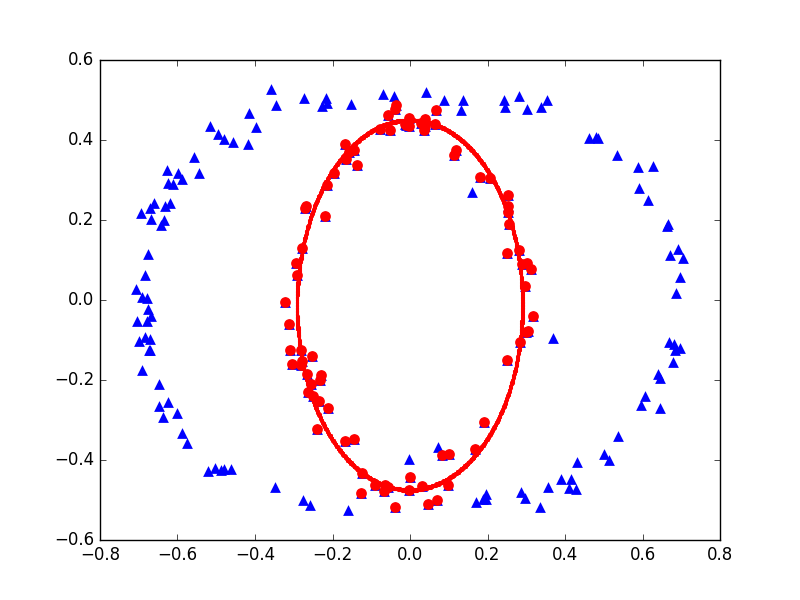}
  \end{minipage}

      \caption{Different $(\mu,\delta)$-labels obtained when applying Algorithm \ref{algo:labelling} to Figure \ref{fig:two circles}, with $\delta=0.05$.}
	\label{fig:labelled_overlapping}

\end{figure}

Notice that there are some points that weren't assigned any label. This depends on the value of $\delta$ -- there exists a trade-off between selecting a small $\delta$, which could potentially miss labelling some points, and selecting a large $\delta$, which could cause finding false labels.

As a comparison, we applied the spectral clustering algorithm proposed in \citep{SPECTRAL} to the dataset. In this experiment, we tried to find 2, 3 and 4 clusters in the cloud of points. As we see in Figure \ref{fig:spectral}, the output seems quite unsatisfactory, since by definition clustering algorithms try to find patterns by partitioning the data, but as discussed before, the patterns found in this dataset cannot be discovered just by partitioning the cloud of points.

\begin{figure}[h]
  \centering
  \begin{minipage}[b]{0.49\textwidth} 
    \includegraphics[width=0.8\textwidth]{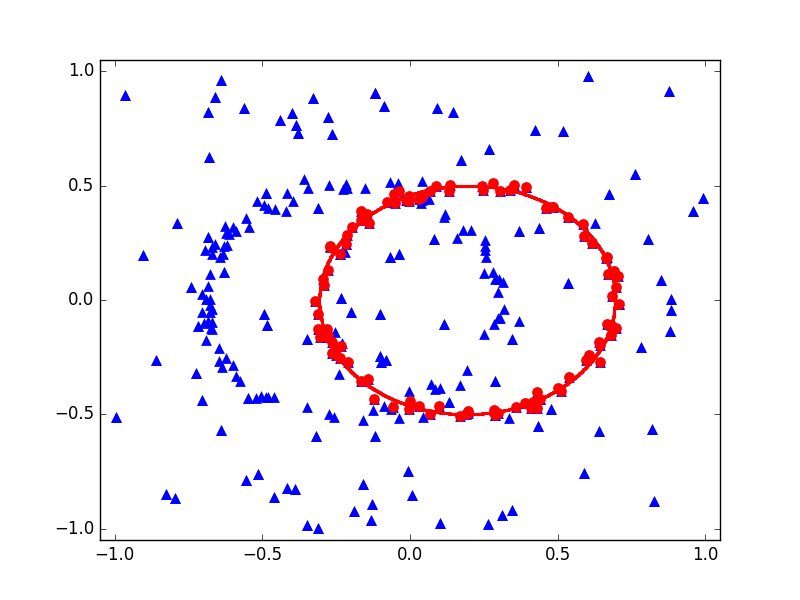}

  \end{minipage}
  \hfill
  \begin{minipage}[b]{0.49\textwidth}
    \includegraphics[width=0.8\textwidth]{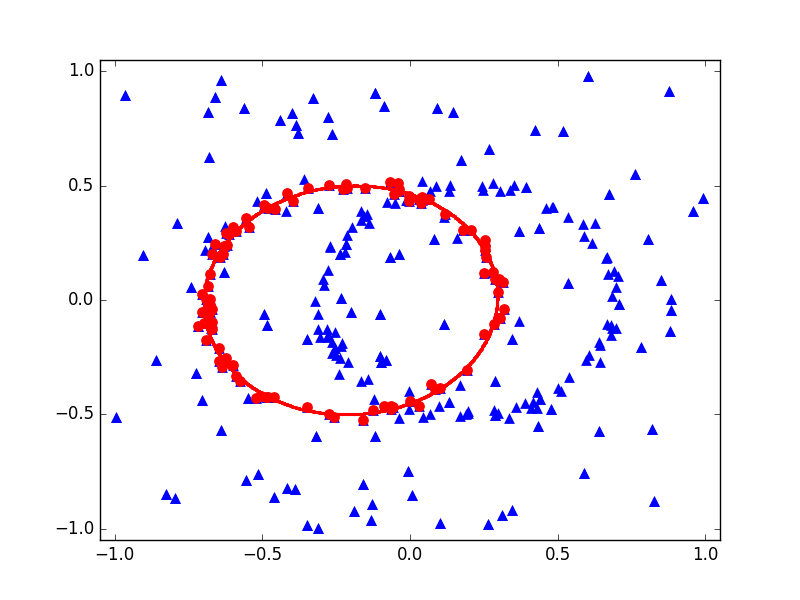}
  \end{minipage}
  \begin{minipage}[b]{0.49\textwidth} 
    \includegraphics[width=0.8\textwidth]{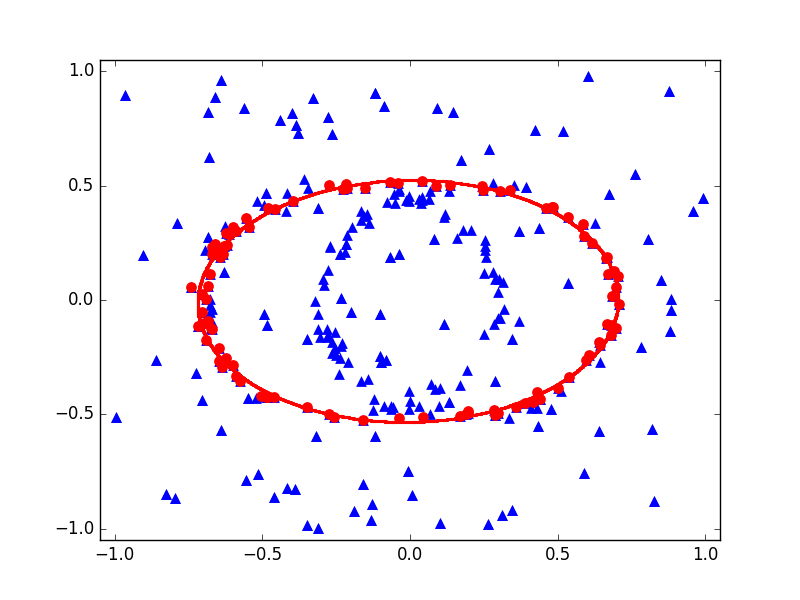}

  \end{minipage}
  \hfill
  \begin{minipage}[b]{0.49\textwidth}
    \includegraphics[width=0.8\textwidth]{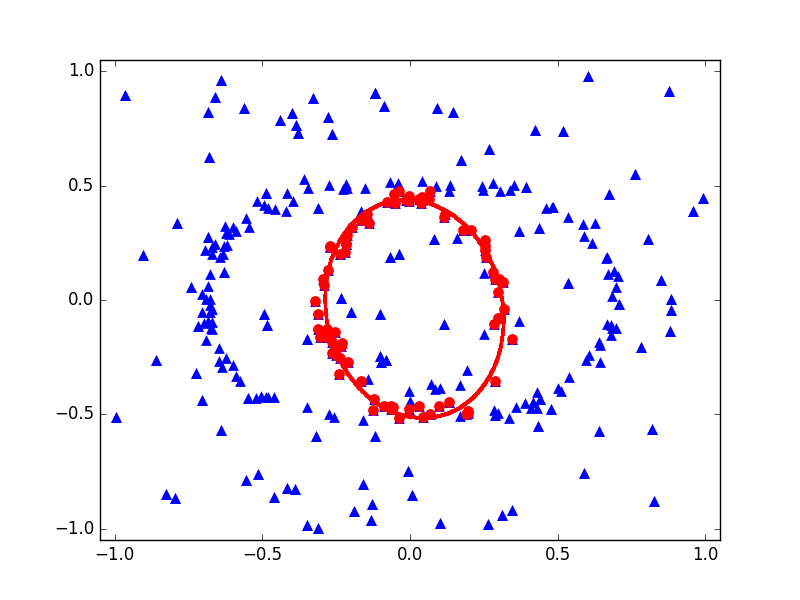}
  \end{minipage}

      \caption{The algorithm performs well even when we add a significant amount of noise. In this new cloud of points $1/3$ of the points are noise.}
	\label{fig:labelled_overlapping_noise}

\end{figure}

\subsubsection{Adding background noise}

As a second experiment, we added background noise to $\mathcal{C}$. More specifically, we added 100 uniformly distributed points to $\mathcal{C}$, so that $1/3$ of the data is now noise. This amount of noise is higher than the one used in previous works on clustering algorithms that deal with noise. For instance, in the celebrated paper \citep{DBSCAN}, where the DBSCAN algorithm is introduced with the aim to to cluster noisy cloud of points, the algorithm is tested with a dataset with 10\% of noise.

As we see in Figure \ref{fig:labelled_overlapping_noise}, the algorithm was robust enough to still identify the same labels as in the noiseless case, without obtaining false discoveries.

\begin{figure}[h]
  \centering
  \begin{minipage}[b]{0.49\textwidth} 
   \includegraphics[width=\textwidth]{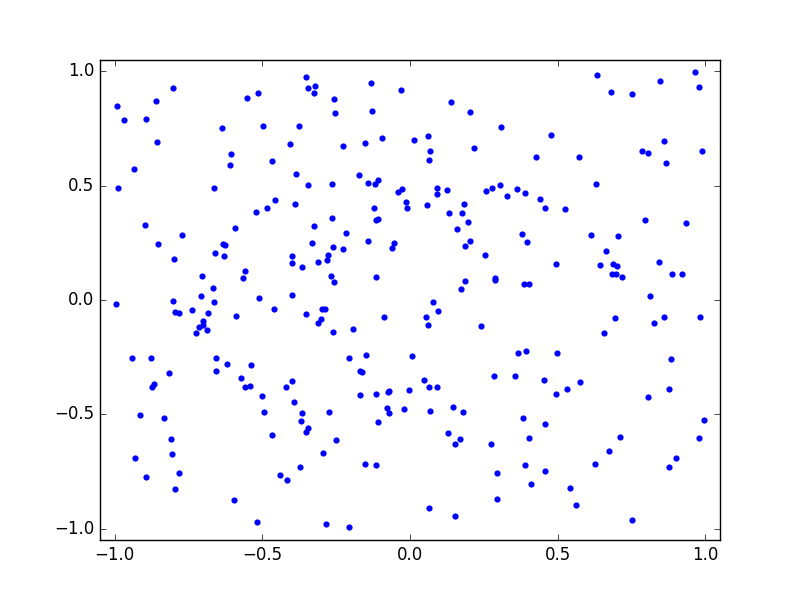}
    \caption{When the signal-to-noise ratio is decreased, the two circles become less apparent making it challenging to the human eye to find the hidden patterns in the data.}
	\label{fig:two circles low signal}
	\end{minipage}
	\hfill
	  \begin{minipage}[b]{0.49\textwidth} 
   \includegraphics[width=\textwidth]{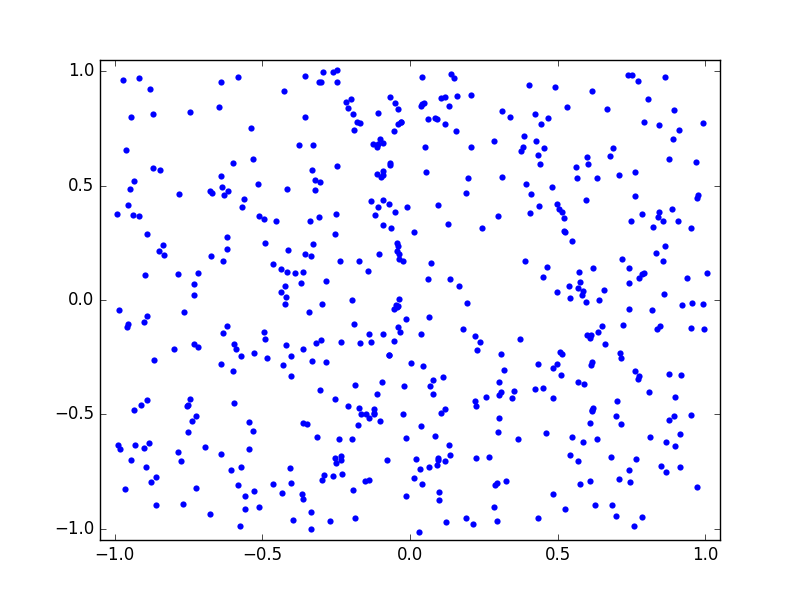}
    \caption{This noisy cloud of points contains three conics, but it is not visually obvious where they are located due to the low signal-to-noise ratio.}
	\label{fig:arbitrary conics low signal}
	\end{minipage}
\end{figure}

\begin{figure}[h]
  \centering
  \begin{minipage}[b]{0.49\textwidth} 
    \includegraphics[width=0.8\textwidth]{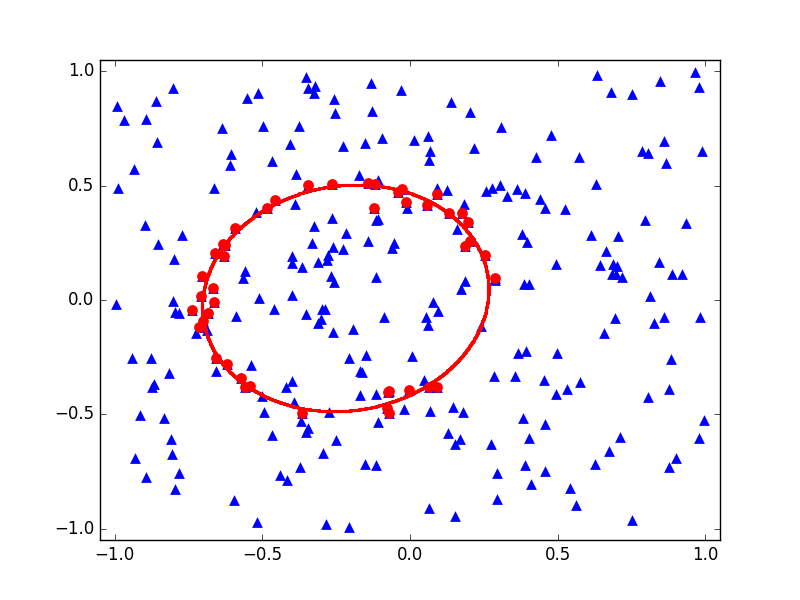}

  \end{minipage}
  \hfill
  \begin{minipage}[b]{0.49\textwidth}
    \includegraphics[width=0.8\textwidth]{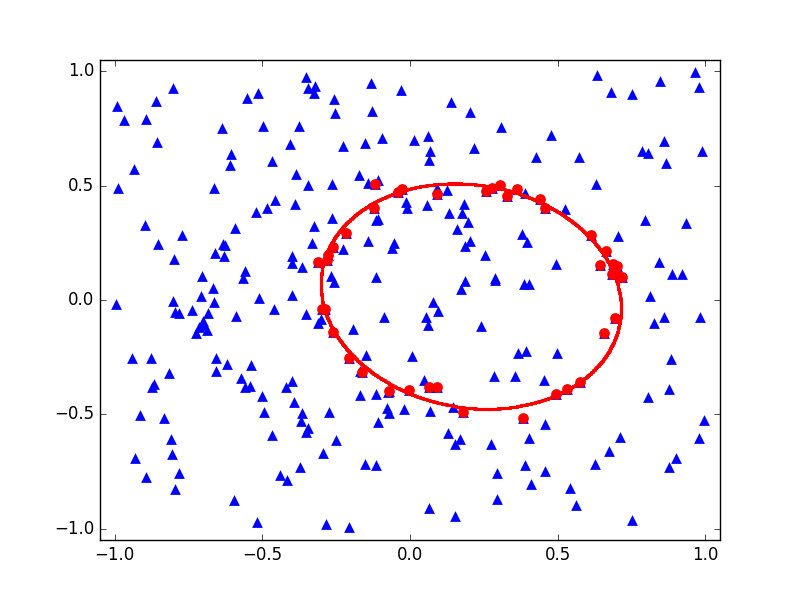}
  \end{minipage}
  \begin{minipage}[b]{0.49\textwidth} 
    \includegraphics[width=0.8\textwidth]{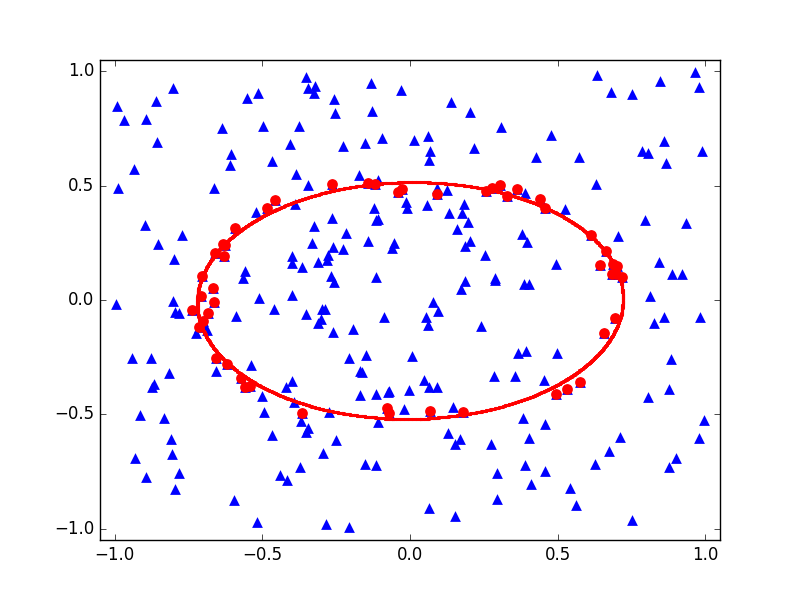}

  \end{minipage}
  \hfill
  \begin{minipage}[b]{0.49\textwidth}
    \includegraphics[width=0.8\textwidth]{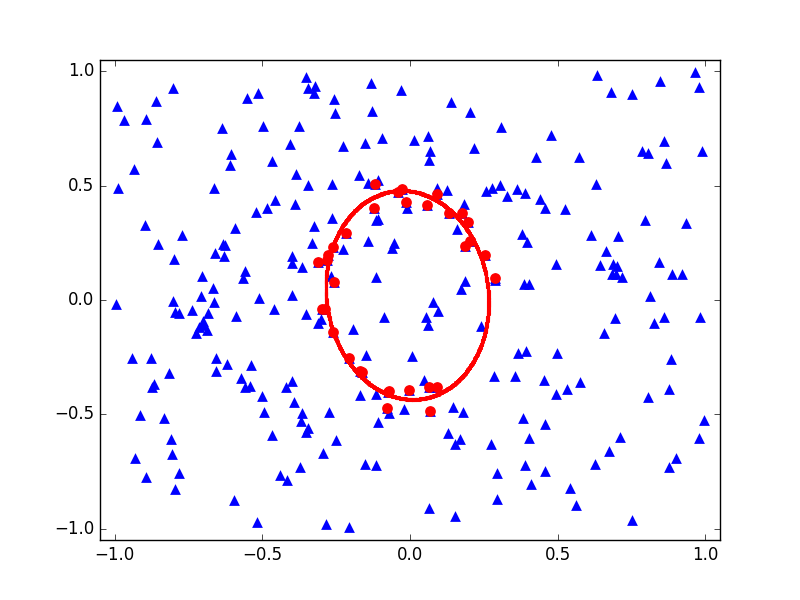}
  \end{minipage}

      \caption{The algorithm still captures the hidden patterns in Figure \ref{fig:two circles low signal} when  the signal-to-noise ratio is significantly decreased.}
	\label{fig:labelled_overlapping_noise_low_signal}

\end{figure}

\begin{figure}[h]
  \centering
   \includegraphics[width=0.5\textwidth]{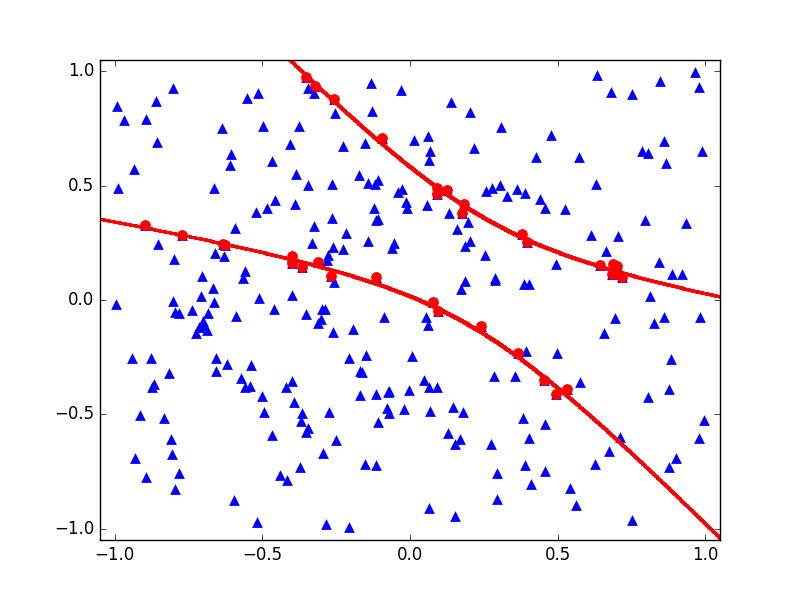}
    \caption{When the amount of noise in the dataset is very high, some false discoveries arise. This false discovery came from Figure \ref{fig:two circles low signal}.}
	\label{fig:two circles false}
\end{figure}

\subsubsection{Small signal-to-noise ratio}	
	
In our third experiment, we significantly reduced the number of points sampled from each circle while keeping the same amount of noise (Figure \ref{fig:two circles low signal}). More specifically, we sampled 40 points from each circle and we then added 200 uniformly distributed points in $[−1, 1]^2$. Therefore, 71.4\% of the dataset is noise, and as a consequence identifying the two circles becomes challenging even for the human eye. However, our algorithm is able to identify the hidden patterns in the data as shown in Figure \ref{fig:labelled_overlapping_noise_low_signal}. However, due to the low signal-to-noise ratio, some false-discoveries also arise, such as the one shown in Figure \ref{fig:two circles false}.

Figure \ref{fig:arbitrary conics low signal} shows three conics hidden in a noisy cloud of points, but the hidden patterns in the data are not visually obvious. Figure \ref{fig:arbitrary conics revealed} shows the three labels hidden in the cloud of points, which correspond to the three conics.

\begin{figure}[h]
  \centering
   \includegraphics[width=0.5\textwidth]{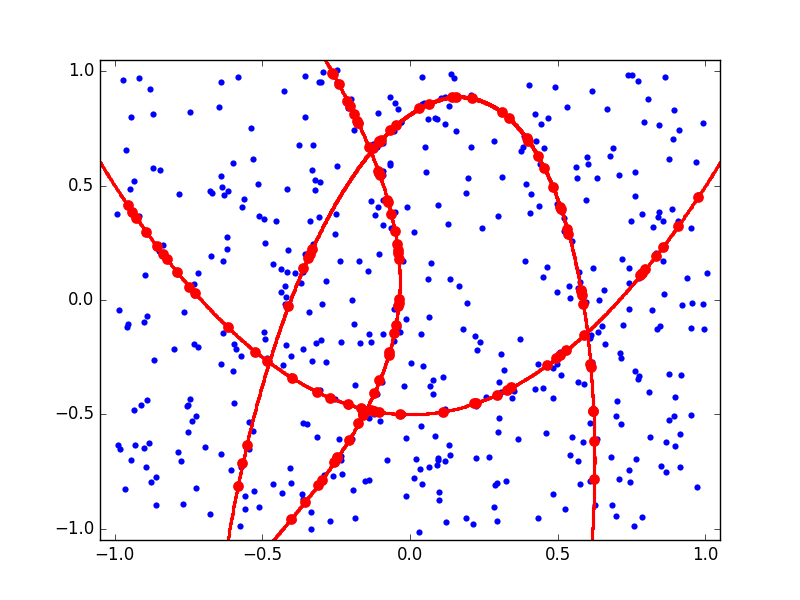}
    \caption{Labels that correspond to the three conics in Figure \ref{fig:arbitrary conics low signal}.}
	\label{fig:arbitrary conics revealed}
\end{figure}

\begin{figure}[h]
  \centering
   \includegraphics[width=0.5\textwidth]{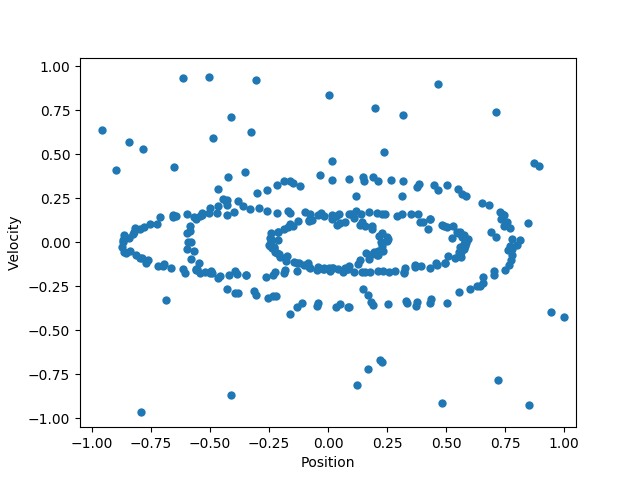}
    \caption{Position and velocity of three simple pendulums with different masses, amplitudes and centres, with background noise.}
	\label{fig:pendulum}
\end{figure}

\begin{table}[h]
\centering
\begin{tabular}{|l|cl|ll|ll|}
\hline
\multicolumn{1}{|c|}{} & \multicolumn{2}{c|}{Centre}                                                                              & \multicolumn{2}{c|}{Amplitude}                                                                           & \multicolumn{2}{c|}{Mass}                                                                             \\
\multicolumn{1}{|c|}{}         & Original                      & \multicolumn{1}{c|}{\begin{tabular}[c]{@{}c@{}}Estimated\\\end{tabular}} & \multicolumn{1}{c}{Original} & \multicolumn{1}{c|}{\begin{tabular}[c]{@{}c@{}}Estimated\end{tabular}} & \multicolumn{1}{c}{Original} & \multicolumn{1}{c|}{\begin{tabular}[c]{@{}c@{}}Estimated\end{tabular}} \\ \hline
Pendulum 1                        &   0.1015                   & 0.1017 & 0.6945                    & 0.6975                                                                           & 3.6181                     & 3.6395 \\
Pendulum 2                         & 0.1703                     & 0.1710                                                                           & 0.4131                    & 0.4134 & 6.1357                    & 6.1563 \\
Pendulum 3                            & -0.3155                     & -0.3129                                                                           & 0.5519                    & 0.5565 & 9.1091                    & 9.4464 \\ \hline
\end{tabular}
\caption{The centre, amplitude and mass of the original three pendulums, and the corresponding quantities estimated from the labels of the points.}
\label{table:pendulum}
\end{table}

\subsection{Recognising simple pendulums from measurements}	\label{subsec:pendulum}

Figure \ref{fig:pendulum} shows a dataset obtained from three (simulated) simple pendulums whose amplitudes, centres and masses are different, with some background noise. Each point in the dataset represents a measurement of the position and velocity of some particles, some of which come from simple pendulums. These measurements could come from an astronomer that is recording the positions of celestial objects for a long period of time -- some of these objects will correspond to planets in the solar system (the pendulums in our case) but there may be some other objects that are recorded such as comets, distant stars, etc (the noise in our case).

If we consider a single point of the dataset, it is impossible to fully characterise the pendulum the measurement came from since two pendulums with different amplitudes, centres and masses could have points that share the same position and velocity. Hence, if one wants to identify all the pendulums hidden in the dataset shown in Figure \ref{fig:pendulum}, one has to consider the dataset as a whole and not just the individual points, since they do not contain any information about the pendulums by themselves.

Therefore, labels arise naturally in this situation: a set of points will share the same label, if they all come from the same pendulum -- in other words, the pendulums themselves are labels.

Using the procedure discussed in this paper, we can label the dataset in Figure \ref{fig:pendulum} and not only identify which points come from the same pendulum, but also fully characterise the three labels identifying the mass, centre and amplitude of the pendulums. As we see in Table \ref{table:pendulum}, one is able to correctly find  the pendulums using the unlabelled dataset, in an unsupervised manner.
	
\section{Conclusion}

In this paper, we develop a framework for the \textit{labelling problem}, an unsupervised learning problem whose objective is to assign one, multiple or no labels to the elements of a given dataset. This should be seen as a dual of the traditional \textit{clustering problem}, where one tries to find a suitable partition of the dataset. In clustering, by definition, elements of the dataset are assigned a single cluster, which is often too restrictive in real applications.

Intuitively, a label is a real-valued function defined on the space where the elements of the dataset are defined. The complexity of the labels will depend on a given feature mapping, which maps the original dataset to a feature space. A cloud of points will be labelled with one of such real-valued functions, if the cloud of points is \textit{unreasonably close} to the kernel of the function, compared to a given probability measure known as the background noise -- a notion that is made precise in Definition \ref{def:label}.

In Section \ref{sec:false discoveries}, we analysed what is the probability of obtaining a false discovery. Analysing the least singular value of some random matrices, we showed in Theorem \ref{th: false discovery} that the probability of obtaining false discoveries decays exponentially with the sample size. We have also proposed an algorithm to find labels in a given dataset (Algorithm \ref{algo:labelling}), which was then tested in Section \ref{sec:experiments} with a synthetic dataset and a dataset of measurements of three simple pendulums. In there, we showed that the algorithm came up with results much more intuitive than what any clustering algorithm can come up with, since clustering algorithms try to partition the data in some suitable way, which is not very natural in many scenarios. As we have shown in Section \ref{sec:experiments}, our algorithm performs well even when we add background noise to the dataset (Figure \ref{fig:labelled_overlapping_noise}) or we reduce the signal-to-noise ratio (Figure \ref{fig:labelled_overlapping_noise_low_signal}).



\acks{This work was supported by The Alan Turing Institute under the EPSRC grant EP/N510129/1. The authors wanted to thank Jiajie Zhang for the invaluable suggestions in the development of this paper.}

\vskip 0.2in
\bibliography{references} 

\end{document}